\documentclass[letterpaper, 10 pt, conference]{ieeeconf}

\usepackage{amsmath}
\usepackage{amssymb}
\usepackage{mathtools}

\usepackage{tabularx}
\usepackage{graphicx}
\usepackage{bm}
\usepackage{color}
\usepackage{float}
\usepackage{units}
\usepackage{url}
\usepackage{hyperref}
\usepackage{balance}
\usepackage{amsmath}
\usepackage{makecell}
\usepackage{cite}
\usepackage{contour}
\usepackage{xcolor}
\usepackage{amsthm}
\usepackage{algorithm, algorithmicx, algpseudocode}
\usepackage{booktabs, multirow, multicol}
\usepackage[table]{xcolor}
\usepackage[dvipsnames]{xcolor}

\definecolor{lightgrayrow}{RGB}{240,240,240}
\definecolor{lightpinkrow}{RGB}{255,230,235}

\newcommand{\outline}[1]{%
  \contourlength{1pt}
  \contour{black!75}{\textbf{\textcolor{white}{#1}}}%
}

\newcommand{\outlineText}[1]{%
  \contourlength{0.7pt}
  \contour{black!75}{\textbf{\textcolor{white}{#1}}}%
}

\makeatletter 
\def\endfigure{\end@float} 
\def\endtable{\end@float}
\makeatother

\usepackage{subcaption}
\usepackage[]{graphicx}
\usepackage{caption}

\newcommand{\greytext}[1]{\textcolor{gray}{#1}}

\newcommand{\greentext}[1]{\textcolor{green}{#1}}
\newcommand{\todo}[1]{\textcolor{red}{#1}}

\newtheorem{remark}{\textbf{Remark}}
\newtheorem{theorem}{Theorem}

\newtheorem{lemma}{Lemma}

\newtheorem{assumption}{Assumption}
\newtheorem{proposition}{Proposition}

\IEEEoverridecommandlockouts

\begin{document} 

\title{\Large \bf 			
\outline{MIRROR}: Visual \outlineText{M}otion \outlineText{I}mitation via \outlineText{R}eal-time \outlineText{R}etargeting and Tele\outlineText{O}pe\outlineText{R}ation with Parallel Differential Inverse Kinematics
}


\author{Junheng Li, Lizhi Yang, and Aaron D. Ames\thanks{Authors are with the Department of Mechanical and Civil Engineering, California Institute of Technology, Pasadena, USA.}\thanks{Email:{\tt\small $
\{$junhengl,lzyang,ames$\}$@caltech.edu}}\thanks{This work is in part supported by Technology Innovation Institute, The Dow Chemical Company project \#227027AW, and Westwood Robotics.}
\\ \vspace{0.2cm} 
\textit{\scriptsize Project Website: \url{https://caltech-amber.github.io/mirror/}}
\vspace{-1.2cm}												
}	
\maketitle


\begin{abstract}
Real-time humanoid teleoperation requires inverse kinematics (IK) solvers that are both responsive and constraint-safe under kinematic redundancy and self-collision constraints. While differential IK enables efficient online retargeting, its locally linearized updates are inherently basin-dependent and often become trapped near joint limits, singularities, or active collision boundaries, leading to unsafe or stagnant behavior. We propose a GPU-parallelized, continuation-based differential IK that improves escape from such constraint-induced local minima while preserving real-time performance, promoting safety and stability. Multiple constrained IK quadratic programs are evaluated in parallel, together with a self-collision avoidance control barrier function (CBF), and a Lyapunov-based progression criterion selects updates that reduce the final global task-space error. The method is paired with a visual skeletal pose estimation pipeline that enables robust, real-time upper-body teleoperation on the THEMIS humanoid robot hardware in real-world tasks. 

\end{abstract}


\section{Introduction}
\label{sec:Introduction}

Humanoid robots are increasingly expected to operate in human-centered environments, where intuitive teleoperation, demonstration-based control, and human motion imitation are essential for natural and human-like interaction \cite{gu2025humanoid}. 
%
The predominant teleoperation pipelines are typically wearable-equipment-based, such as VR headsets \cite{bertrand2024high, he2024omnih2o, li2025clone}, motion-capture suits \cite{ze2025twist}, and haptic feedback devices \cite{colin2023whole, he2025novel, dafarra2024icub3}, allowing rich feedback and tailored interfaces. However, these systems require users to wear bulky hardware, undergo repeated calibration procedures, and are often limited to prepared lab environments.  This is a bottleneck to the widespread deployment of teleoperated humanoids, and the data-generation pipeline they would enable. 


Recent advances in optical sensors have made it possible to extract reliable human pose data in real-time. He et al. \cite{he2024learning} leverages a monocular camera for pose estimation and demonstrates the feasibility of real-time humanoid teleoperation with only RGB cameras. Lately, stereo cameras have made 3-D pose estimation more reliable at high frequency \cite{zago20203d}, creating new opportunities for perceptive humanoid retargeting and teleoperation with better pose depth estimation. However, deploying such a pipeline in real-world teleoperation remains challenging. The optical setup must be portable, low-latency, and robust to (partial) occlusion \cite{darvish2023teleoperation}.
The whole-body kinematic retargeting pipeline must be computationally lightweight, collision-aware, globally optimal, and capable of operating at control frequencies suitable for real-time whole-body humanoid control. 

In this paper, we present MIRROR, a hierarchical, self-contained, low-latency humanoid retargeting and teleoperation pipeline using stereo cameras and a novel GPU-accelerated differential inverse kinematics (IK) framework.

\begin{figure}[!t]
\vspace{0cm}
    \center	
    \includegraphics[clip, trim=0.0cm 0.0cm 0.0cm 0.0cm, width=1\columnwidth]{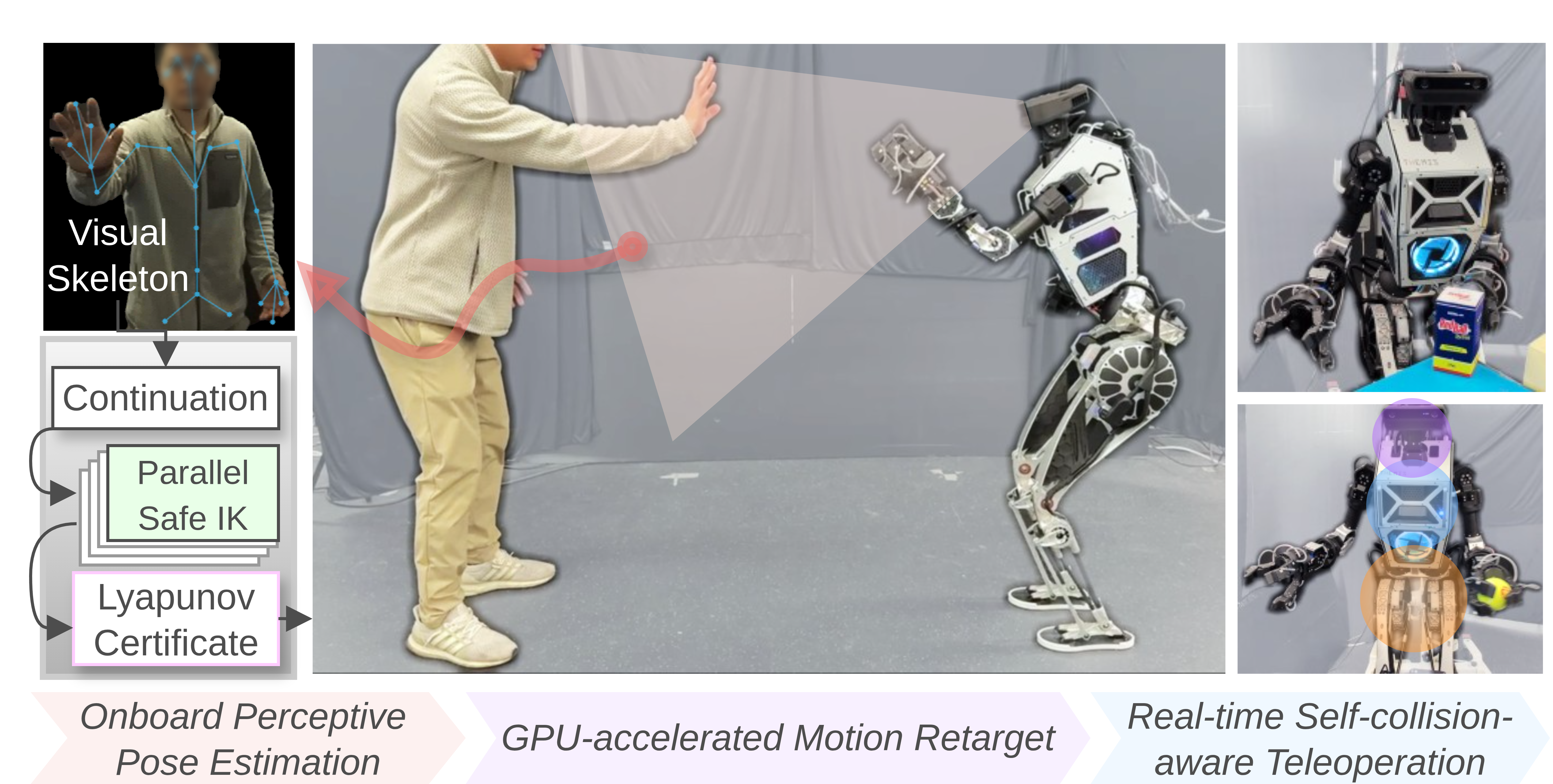}
    \caption{{\bfseries Real-time human-humanoid motion mirroring teleoperation.} 
    Full experiment videos \url{https://youtu.be/LvwIwOTTu9g}}
    \label{fig:hero}
    \vspace{-0.5cm}
\end{figure}

\begin{figure*}[!t]
    \center	
    \includegraphics[clip, trim=0.2cm 0.5cm 0.2cm 0cm, width=2\columnwidth]{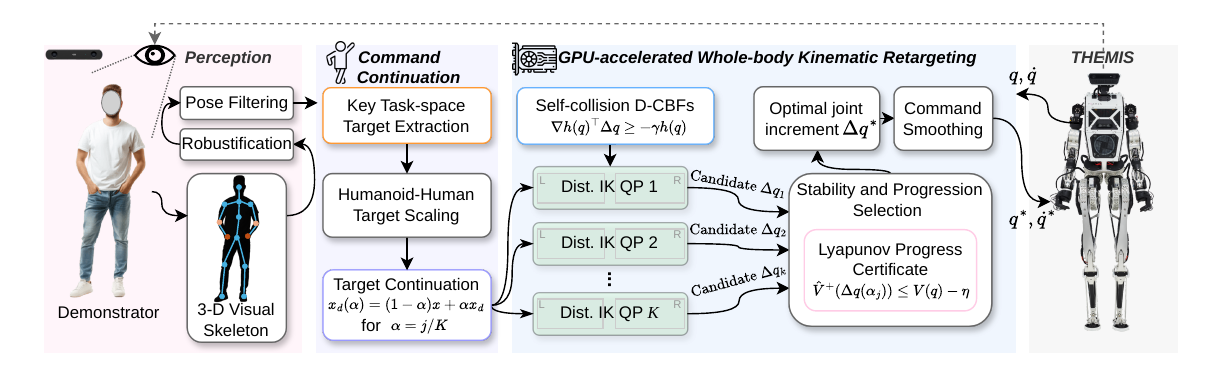}
    \caption{{\bfseries MIRROR pipeline architecture.} }
    \label{fig:system}
    \vspace{-0.2cm}
\end{figure*}

Inverse kinematics (IK) underpins motion control and whole-body retargeting by mapping lower-dimensional task-space commands to higher-dimensional joint configurations. Classical global IK formulates this mapping as a nonlinear optimization over joint space, achieving accurate solutions but at substantial computational cost, making it more suitable for offline retargeting pipelines \cite{darvish2023teleoperation, yang2025omniretarget, araujo2025retargeting}. In contrast, differential IK solves a constrained quadratic program (QP) over joint velocity increments, enabling high-frequency control and widespread adoption in humanoid whole-body frameworks \cite{del2016implementing, pink}. However, despite its efficiency, differential IK is inherently local and basin-dependent, often becoming trapped in local minima when task constraints are underdetermined or when commands deviate significantly from nominal configurations.
Kobayashi and Jin \cite{kobayashi2023mirror} attempt to allow Jacobian-based IK to escape local minima via a heuristics-based $\epsilon$-clamping. Giamou et al. \cite{giamou2022convex} leverages iterations of lightweight low-rank convex optimization. Nonetheless, a good middle ground between global convergence and real-time feasibility is highly sought after.

Recently, parallelization has been explored to accelerate optimization and planning. GPU-based optimization \cite{bishop2024relu, jeon2024cusadi}, parallel sampling in model predictive control \cite{xue2025full,alvarez2025real}, and parallelizable physics simulations \cite{todorov2012mujoco} are examples where parallel computation improves robustness or speed in robotic control. In kinematic retargeting, however, popular toolkits still rely on single-instance differential IK solvers \cite{pink, Zakka_Mink_Python_inverse_2026}. CuRoBo \cite{sundaralingam2023curobo} and HJCD-IK\cite{yasutake2025hjcd} attempt to solve global IK via GPU parallelization for accelerated motion generation. IKFlow \cite{ames2022ikflow} leverages deep neural networks to produce a variety of IK solutions in real-time. 
While effective on lower-dimensional manipulators, extending these methods to high-DoF systems is challenging due to the increased complexity of the optimization landscape, where strong joint coupling and constraint interactions create narrow feasible manifolds and poorly conditioned local basins. Rather than merely improving computational efficiency, distributed optimization decomposes the problem into smaller structured subspaces, reducing effective exploration complexity and limiting cross-joint interference. Such decompositions have shown improved robustness and solution quality under fixed time budgets in legged systems \cite{ma2020coupled, amatucci2024accelerating}.

This work addresses the fundamental tension between global optimality and real-time feasibility in motion retargeting with modern GPU parallelism. While global IK can escape local minima, it is computationally too expensive for high-frequency control. Differential IK is fast but can become trapped in poor local solutions. We propose a GPU-accelerated Distributed Differential IK method that runs many parallel instances with task command continuation and has a higher probability of selecting the local-basin-escaping solution in real time with increased parallel samples.  As a result, the proposed method significantly improves robustness to local minima while maintaining real-time feasibility.
The main contributions include:
\begin{itemize}
    \item We propose a self-contained, real-time humanoid retargeting and teleoperation pipeline that uses stereo cameras and 3-D visual skeletal pose estimations. 
    \item We introduce an optimization-based, GPU-accelerated parallel IK scheme. We formalize and validate that the proposed parallel differential IK with the task-space continuation and Lyapunov-based certification encourages the escape from the local basin of attraction and pose stagnation, which has been an open challenge in traditional differential IK. 
    \item Our approach employs a dual-layer safety mechanism for self-collision avoidance: control barrier function (CBF) constraints are enforced within the IK optimization, while parallel evaluation further improves closed-loop safety by increasing the likelihood of selecting collision-free solutions.
    \item We validate the proposed pipeline on hardware---specifically the THEMIS humanoid robot---with upper-body motion tracking from a human operator. The proposed low-latency pipeline enables stable motion tracking teleportation under realistic humanoid tasks. 
\end{itemize}


This paper, therefore, demonstrates end-to-end teleoperation on humanoid robots through the use of parallel distributed differential IK---which includes the Lyapunov-like certificates for safety and stability, setting the stage for (principled) integration into broader autonomy stacks.

\section{Pipeline and Approach}
\label{sec:approach}
We present a real-time motion retargeting  and teleoperation pipeline that transforms
demonstrator observations into safe, task-consistent joint commands for humanoid robots. The proposed architecture, illustrated in Fig.~\ref{fig:system}, decomposes the pipeline into three modular stages, enabling high-frequency motion generation while maintaining explicit safety constraints and improving robustness to local basin trapping in constrained differential IK.

\subsection{Perception and Pose Processing}
\paragraph{Visual Skeleton} 
The perception module estimates the demonstrator’s full-body skeletal structure from stereo perception.
In particular, the front end runs the Stereolabs ZED 2i stereo camera and its Body Tracking module to extract \texttt{Body38} 3-D visual skeleton representation, using it to infer
$n$ key task-space targets $\mathcal{P}=\{ p_t\in \mathbb{R}^3,\: t = 0,1,...,n\}$ relevant to humanoid motion retargeting. These include, but are not limited to, torso position, hand positions, elbow positions, and fingertip positions. The tracked keypoints are converted into a local body frame by subtracting the estimated body reference point for a consistent coordinate system in downstream task extraction.

\paragraph{Pose Filtering and Robustification}
Because vision-based keypoints can be intermittent and exhibit spurious outliers, we apply a lightweight filtering stage inside the perception module before publishing targets to the shared state. Specifically, each
3-D keypoint is processed by a low-pass filter with three robustness features:
(i) \textbf{confidence-based gating}, (ii) \textbf{jump rejection}, and (iii) \textbf{multi-body rejection}. Confidence-based gating
ensures both numerical integrity (\textit{e.g.}, NaNs) and tracker confidence. Jump rejection detects large frame-to-frame deviations using an $\ell_2$ threshold to suppress transient spikes. In the multi-body rejection module, candidate bodies are evaluated based on spatial proximity to the previously tracked target. The closest valid detection is selected as the current demonstrator, while other detections are rejected. This prevents identity switching in multi-person environments.

Concretely, for a measured keypoint $p_t \in \mathbb{R}^3$ with a filtered state
$\bar p_{t-1}$, we compute
\begin{equation}
\bar p_t =
\begin{cases}
\bar p_{t-1}, & \text{if invalid },\\[2pt]
\lambda_\textrm{jr} p_t + (1-\lambda_\textrm{jr})\bar p_{t-1}, & \text{if } \|p_t-\bar p_{t-1}\|>\tau_\textrm{jr},\\[2pt]
\alpha_\textrm{jr} p_t + (1-\alpha_\textrm{jr})\bar p_{t-1}, & \text{otherwise},
\end{cases}
\end{equation}
where $\tau_\textrm{jr}$ is the jump threshold and $\lambda_\textrm{jr} \ll \alpha_\textrm{jr}$ is a conservative
mixing coefficient used only during jump rejection.

\subsection{Task-space Command Continuation}
\paragraph{Command Scaling} Directly solving inverse kinematics to track $\mathcal{P}$ can lead to infeasible or unstable joint updates due to kinematic differences between the robot and the operator. Each key body-frame task-space command's position part is then scaled by $\beta_t$:
\begin{align}
    x_d^\textrm{pos} = \beta_t (p_t-p_c) + x_c,
\end{align}
where $p_c$ denotes the skeleton body anchor position and $x_c$ denotes the robot torso CoM feedback.

\paragraph{Command Continuation} 
In practice, large task-space displacements may require joint motions that violate safety constraints or reach singularities, resulting in convergence to undesirable local solutions, which has been a challenge for traditional differential IK approaches.

To mitigate this issue, we introduce a task-space
continuation mechanism that gradually interpolates the desired task command from the current robot state via parallel solvers at each time step.
For a humanoid whole-body configuration space
$q \in \mathbb{R}^{n}$,
let $x:=x(q)$ denote the current task value and $x(\cdot)$ denote the forward kinematics map. For $\alpha\in[0,1]$, define the continuation target
\begin{equation}
x_d(\alpha) := (1-\alpha)\,x + \alpha\,x_d.
\end{equation}
Then the intermediate displacement from the current task is
\begin{equation}
d_x(\alpha) := x_d(\alpha) - x(q) = \alpha\,e(q).
\end{equation}
We select a set of $K+1$ continuation parameters $\{\alpha_j\}_{j=0}^K\subset[0,1]$.
In our implementation, we use a deterministic grid $\alpha_j=j/K$; alternatively, one may sample $\alpha_j$
randomly for probabilistic coverage guaranties (Sec.~\ref{subsec:escape_prob}).

At each control step, we evaluate $K$ candidate joint increments corresponding to different continuation parameters $\alpha$, thereby reducing the likelihood of constraint-dominated stagnation that arises when tracking $x_d$ directly. A detailed probabilistic basin-escape proof is presented in Sec. \ref{sec:pdik}.

\subsection{GPU-accelerated Whole-body Kinematic Retargeting}

\paragraph{Self-collision CBFs}
To enforce self-collision avoidance between upper limbs and torso/head (e.g., in \cite{khazoom2022humanoid}),
we model the body as a set of bounding spheres with centers
$c_o(q) \in \mathbb{R}^3$ and radii $R_o$.
For a tracked point $x_i^\textrm{pos}(q) \in \mathbb{R}^3$
(\textit{e.g.}, wrist or elbow),
define the safety function
\begin{equation}
\label{eq:cbf_h}
h_{i}(q) = \| x_i^\textrm{pos}(q) - c_o(q) \|^2 - \rho_{i}^2,
\end{equation}
where $\rho_{i} = R_o + r_i + m$
includes the sphere radius, limb thickness, and safety margin $m$.
Safety requires $h_{i}(q) \ge 0$.

Using first-order linearization,
\[
h_{i}(q + \Delta q) \approx h_{i}(q)+ \nabla h_{i}(q)^\top \Delta q,
\]
where
\[
\nabla h_{i}(q)^\top = 2 (x^\textrm{pos}_i - c_o)^\top
\left( J_{x_i}(q) - J_{c_o}(q) \right).
\]

A discrete-time Control Barrier Function (D-CBF) condition is imposed \cite{ames2016control} :
\begin{equation}
\label{eq:cbf_constraint}
\nabla h_{i}(q)^\top \Delta q
\ge
-\gamma\, h_{i}(q),
\qquad \gamma > 0,
\end{equation}
which ensures forward invariance of safe configurations under bounded increments. In this work, we embed 4 key locations (hands and elbows) w.r.t. 3 CBF spheres forming the robot's upper body (shown in Fig.~\ref{fig:hero}). 

\paragraph{Family of Parallel QPs}
For each $\alpha_j$, we solve in parallel the distributed convex QP for each arm chain
\begin{align}
\Delta q(\alpha_j) \in \arg\min_{\Delta q}\quad
& \frac{1}{2}\left\|J(q)\Delta q - \alpha_j e(q)\right\|_{W_x}^2
+ \frac{1}{2}\|\Delta q\|_{W_q}^2
\label{eq:pc_qp_obj}\\
\text{s.t.}\quad
& \nabla h_{i}(q)^\top \Delta q \ge -\gamma h_{i}(q), \label{eq:cons1}
\\
& \Delta q_{\min}\le \Delta q \le \Delta q_{\max},  \label{eq:cons2}
\\
& q_\textrm{min} \leq q + \Delta q \leq q_\textrm{max}, \label{eq:cons3}
\end{align}
where $ e(q) := x_d - x(q) $ is the final-goal error, $J(q)$ is the corresponding task Jacobian, and $W_q\succ 0$, $W_x\succ 0$ are diagonal weighting matrices.
The objective is strongly convex and each feasible QP admits a unique minimizer.


\paragraph{Lyapunov Progress Certificate and Selection}
Given a desired target $x_d$ and error $e(q)$, define a Lyapunov-like objective based on the global IK tracking goal
\begin{equation}
\label{eq:lyapunov_pc}
V(q) = \frac{1}{2} e(q)^\top W\, e(q).
\end{equation}

Using first-order kinematics,
\begin{equation}
\label{eq:firstKin}
x(q+\Delta q) \approx x(q) + J(q)\Delta q,
\end{equation}
the predicted next final-goal error is
\begin{equation}
\hat e^{+}(\Delta q) = x_d - \bigl(x(q)+J(q)\Delta q\bigr) = e(q) - J(q)\Delta q,
\end{equation}
and the predicted next Lyapunov value is
\begin{equation}
\hat V^{+}(\Delta q) = \frac{1}{2}\hat e^{+}(\Delta q)^\top W\, \hat e^{+}(\Delta q).
\end{equation}
We accept a candidate $\alpha_j$ if it is feasible and yields sufficient predicted progress:
\begin{equation}
\hat V^{+}\bigl(\Delta q(\alpha_j)\bigr) \le V(q) - \eta,
\label{eq:stable_accept}
\end{equation}
for some $\eta>0$.
We additionally require an escaping ``reverse'' update away from QP convergence (\textit{ e.g.},
$\|\Delta q(\alpha_j)\|\ge \varepsilon_q$ when $V(q)>\varepsilon_V$). This encourages staying away from intermediate task goals to reduce the chances of getting stuck at a local basin of attraction.
Among accepted candidates, we choose the \emph{largest} continuation parameter:

\begin{align}
j^\star &= \max\Bigl\{j:\ \Delta q(\alpha_j)\ \text{satisfies  (\ref{eq:cons1}-\ref{eq:cons3}, \ref{eq:stable_accept})}\Bigr\}, \\
 q^* &= q^+=q +\Delta q(\alpha_{j^\star}), \\
 \dot q^* &= \Delta q(\alpha_{j^\star})/dt.
\label{eq:max_alpha_select}
\end{align}

\begin{remark}
\textup{
This selection yields the most aggressive (closest-to-target) continuation update that still satisfies
a final-goal progress certificate, which filters candidate QP solutions based on whether they reduce the real nonlinear IK objective to avoid a local basin while still ``staying close" to intermediate targets (\textit{i.e.}, local differential IK objective).
This selection does not improve the intermediate tracking error significantly but strongly certifies progress towards the final goal and away from local minima, seen in Table \ref{tab:multistep_comparison}.}
\end{remark}

\subsection{Robot Hardware}
The final stage of the pipeline sends the selected joint-space commands to execute on the THEMIS humanoid platform from Westwood Robotics. 
These joint position and velocity commands are transmitted to the robot’s low-level Whole-body Control and joint-space PD control.
This decouples perception-driven retargeting from hardware execution while preserving real-time safety and responsiveness. 
By maintaining a hierarchical structure, the framework ensures that kinematic retargeting remains agnostic to the robot's low-level control, promoting modularity and adaptability.


\section{Parallel Differential Inverse Kinematics}
\label{sec:pdik}

In this section, we formalize the basin-dependent behavior of standard differential inverse IK and introduce a parallel task continuation scheme. Specifically, we show that the iterates of differential IK remain confined to local basins of attraction and that evaluating multiple continuation targets in parallel increases the coverage of feasible descent directions within a fixed computation budget. Therefore, the probability of escaping a local basin of attraction (\textit{e.g.}, stagnation and CBF violations) will increase accordingly.

\subsection{Local Basin Dependence of Differential IK}

\begin{lemma}[\textit{Differential IK as a Gauss--Newton step}]
\label{lem:dik_gn}
\textup{Consider the regularized nonlinear IK objective
\begin{equation}
\label{eq:nonlinear_ik_obj}
F(q) = \frac{1}{2}\|x(q)-x^d\|_{W_x}^2 + \frac{1}{2}\|q-q_{\mathrm{ref}}\|_{W_q}^2,
\end{equation}
At an iterate $q_k$, we can linearize by eqn. \eqref{eq:firstKin}. Then it yields the convex QP \eqref{eq:pc_qp_obj}, which computes a Gauss--Newton step for minimizing $F(q)$ at $q_k$ and has a global minimizer $\Delta q_k$ for the local model.
}
\end{lemma}

\begin{proposition}[\textit{Basin dependence of differential IK}]
\label{thm:local_min_trap}
Assume $x(\cdot)$ is twice continuously differentiable on an open set $\mathcal{D}$ and $J_i(q)$ is locally Lipschitz on $\mathcal{D}$. Consider the iterative differential IK update
\begin{equation}
\label{eq:dik_update}
q_{k+1} = q_k + t_k \Delta q_k,
\end{equation}
where $\Delta q_k$ is obtained by solving \eqref{eq:pc_qp_obj} and $t_k\in(0,1]$ may be chosen by a standard line-search or trust-region rule such that $F(q_{k+1}) \le F(q_k)$ whenever $q_k\in\mathcal{D}$. Then every accumulation point of $\{q_k\}$ is a stationary point of $F$ in $\mathcal{D}$.

Furthermore, suppose $F$ has at least two isolated strict local minimizers $q^\star_1$ and $q^\star_2$ in $\mathcal{D}$. Then there exist neighborhoods $U_1$ of $q^\star_1$ and $U_2$ of $q^\star_2$ such that if $q_0\in U_j$ for some $j\in\{1,2\}$, the sequence $\{q_k\}$ remains in $U_j$ and converges to a stationary point in $U_j$ (typically $q^\star_j$). In particular, the method cannot transition from the basin of attraction of one local minimizer to another without violating the descent condition (\textit{i.e.}, without increasing $F$ at some iteration).

\end{proposition}

\begin{remark}
\textup{Since $F(q)$ is generally nonconvex for humanoid kinematics and often admits multiple solution branches (\textit{e.g.}, elbow up vs. down), Proposition~\ref{thm:local_min_trap} implies that differential IK is inherently basin-dependent and may converge to suboptimal local minima dictated by initialization and regularization, known as ''basin-trapping''.}
\end{remark}

\subsection{Basin Escape via Increased Coverage}
\label{subsec:escape_prob}
The continuation strategy is designed to mitigate basin-trapping behavior that can arise when a single-shot update (\textit{i.e.}, differential IK with $\alpha=1, \:t_k = 1$ only) repeatedly produces steps constrained by the same active set (\textit{e.g.}, joint locks, singularity, CBF facets),
leading to stagnation.

\paragraph{Progress-certified Escape}
Let $F(q)$ denote the global regularized IK objective in \eqref{eq:nonlinear_ik_obj}.
We define an objective-based escape event for a candidate step $\Delta q$:
\begin{equation}
\mathcal{E}(q,\Delta q) :=
\left\{F(q+\Delta q)\le F(q)-\eta_F\right\},
\label{eq:escape_event}
\end{equation}
with $\eta_F>0$ above numerical tolerance. In our system, we additionally
require the Lyapunov progress certificate \eqref{eq:lyapunov_pc} to avoid
selecting constraint-dominated stagnant candidates; hence, we consider the \emph{progress-certified escape} event
\begin{equation}
\mathcal{E}_{\mathrm{pc}}(q,\Delta q) := \mathcal{E}(q,\Delta q)\ \cap\ \mathcal{P}(q,\Delta q).
\label{eq:escape_event_pc}
\end{equation}

\paragraph{Escaping $\alpha$-set}
For a fixed $q$, define the set of continuation parameters that would yield a progress-certified escaping update:
\begin{equation}
\mathcal{A}^{\mathrm{pc}}_{\mathrm{esc}}(q) :=
\left\{\alpha\in[0,1]:\ \mathcal{E}_{\mathrm{pc}}(q,\Delta q(\alpha))\ \text{holds}\right\}.
\end{equation}
Let $p_{\mathrm{pc}}(q):=\mu(\mathcal{A}^{\mathrm{pc}}_{\mathrm{esc}}(q))\in[0,1]$ be its Lebesgue measure \cite{bartle2014elements}.
Although our implementation uses a deterministic grid, it is useful to state a randomized continuation variant formalizing the benefit of increasing the number of candidates $K$.

\begin{figure*}[!t]
\vspace{0.2cm}
    \center	
    \includegraphics[clip, trim=1cm 10.5cm 1cm 0cm, width=2\columnwidth]{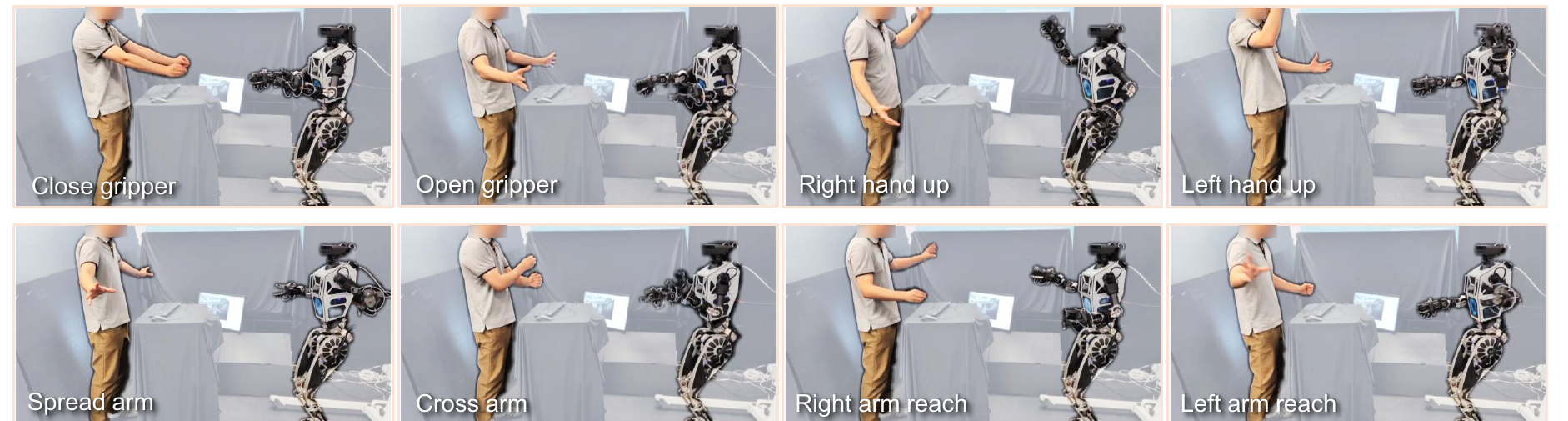}
    \caption{{\bfseries Snapshots of real-time teleoperation.} Pose mirroring and mimicking. }
    \label{fig:mirroring}
    \vspace{-0.2cm}
\end{figure*}

\begin{theorem}[Progress-certified escape probability increases with $K$]
\label{thm:escape_prob_pc}
Fix a state $q$ and draw $\alpha_1,\dots,\alpha_K$ i.i.d.\ from $\mathrm{Unif}[0,1]$.
If $p_{\mathrm{pc}}(q)>0$, then the probability that at least one of the $K$ candidates is progress-certified escaping satisfies
\begin{equation}
\mathbb{P}\Bigl(\exists k\le K:\ \alpha_k\in\mathcal{A}^{\mathrm{pc}}_{\mathrm{esc}}(q)\Bigr)
= 1-(1-p_{\mathrm{pc}}(q))^K,
\end{equation}
which strictly increases in $K$ and converges to $1$ as $K\to\infty$.
\end{theorem}

\begin{proof}
Each $\alpha_k$ lies in $\mathcal{A}^{\mathrm{pc}}_{\mathrm{esc}}(q)$ with probability $p_{\mathrm{pc}}(q)$ under the uniform distribution. Independence yields
$\mathbb{P}(\forall k,\alpha_k\notin\mathcal{A}^{\mathrm{pc}}_{\mathrm{esc}}(q))=(1-p_{\mathrm{pc}}(q))^K$.
Taking the complement gives the claim.
\end{proof}

\begin{remark}
\textup{
A deterministic grid $\alpha_j=j/K$ does not yield the i.i.d.\ expression in Theorem~\ref{thm:escape_prob_pc},
but it provides uniform coverage over $[0,1]$. Empirically, increasing $K$ increases the likelihood of encountering different active-set regimes in the parametric QP solutions, improving basin-escape behavior.}
\end{remark}

\subsection{Distributed QPs and Increased Escape Likelihood}
\label{subsec:distributed_escape}

We now contrast a whole-body IK-QP in $\mathbb{R}^n$ with a distributed formulation that solves multiple smaller subproblems. In our upper-body implementation, we solve two arm QPs with a substantially smaller variable dimension per QP.

Let $\Delta q_R\in\mathbb{R}^{d}$ and $\Delta q_L\in\mathbb{R}^{d}$ denote the right and left arm decision variables,
where $d < n$. The composed full update is
\begin{equation}
\Delta q = P_R \Delta q_R + P_L \Delta q_L,
\end{equation}
where $P_R,P_L$ are fixed embedding matrices. (\textit{i.e.}, with the floating base being shared, $P_R$ and $P_L$ overlap on those DOFs and resolve this by averaging.)
Each subproblem solves a continuation QP \eqref{eq:pc_qp_obj} with the corresponding Jacobians and constraints for that segment.

\noindent \textbf{1) Compute--coverage advantage:}
\begin{lemma}[Escape probability increases with $K$]
\label{lem:escape_vs_K}
\vspace{-0.2cm}
Assume each candidate has probability at least $p$ of satisfying an escape certificate,
and candidate trials are independent (or weakly dependent).
Then the probability of obtaining at least one escaping candidate from $K$ evaluations is at least
$1-(1-p)^K$, which increases with $K$.
\end{lemma}

\begin{proof}
Same as Theorem~\ref{thm:escape_prob_pc}.
\end{proof}

Let $C(m)$ denote the average compute cost of solving a constrained QP of dimension $m$.
For many solvers, $C(m)$ grows superlinearly with $m$ (\textit{e.g.}, cubic for dense factorizations).
Under a fixed wall-clock budget $B$,
\begin{equation}
K_{\mathrm{whole}} \approx \frac{B}{C(n)},\qquad
K_{\mathrm{dist}} \approx \frac{B}{N_s\,C(d)},
\end{equation}
where $N_s$ is the number of segments (here $N_s=2$ arms). Since $d < n$,
we typically have $K_{\mathrm{dist}}> K_{\mathrm{whole}}$.
By Lemma~\ref{lem:escape_vs_K},
this yields a principled compute–coverage advantage: distributed QPs enable more continuation candidates per cycle, which increases the likelihood of encountering an escaping candidate.

\noindent \textbf{2) Dimensionality advantage:}
\begin{assumption}[Local escape neighborhood]
\label{ass:escape_ball}
\vspace{-0.2cm}
At a given state $q$, suppose there exists an escaping feasible increment $\Delta q^\star$ and $r>0$ such that
\begin{equation}
\mathbb{B}_r(\Delta q^\star)\cap \mathcal{S}_{\Delta} \subseteq \mathcal{X}_{\mathrm{esc}},
\end{equation}
where $\mathcal{S}_{\Delta}$ is the feasible increment set and $\mathcal{X}_{\mathrm{esc}}$ is the set of feasible escaping increments under \eqref{eq:escape_event}. An analogous assumption holds in each segment space.
\end{assumption}

\begin{lemma}[Volume ratio decays exponentially with dimension]
\label{lem:vol_ratio}
Let $\mathcal{S}\subset\mathbb{R}^m$ be contained in a ball of radius $R$ and suppose candidates are approximately spread over $\mathcal{S}$. Then the probability of hitting an escaping neighborhood of radius $r$ is lower bounded by a volume ratio:
\begin{equation}
p \ \gtrsim\ \frac{\mathrm{vol}(\mathbb{B}_r)}{\mathrm{vol}(\mathbb{B}_R)}
= \left(\frac{r}{R}\right)^m.
\end{equation}
Thus, for fixed $r/R<1$, the hit probability decays exponentially with dimension $m$.
\end{lemma}
\begin{proof}
Since $\mathcal{S}\subseteq\mathbb{B}_R$, the probability of landing in $\mathbb{B}_r(\Delta q^\star)$ is at least its volume divided by the volume of an enclosing set. The $m$-dimensional ball volume scales as $r^m$.
\end{proof}

Lemma~\ref{lem:vol_ratio} implies that if escaping updates occupy a small ``tube'' or neighborhood in increment space, then reaching them becomes dramatically harder as the dimension increases. Solving distributed QPs in $\mathbb{R}^{d}$
(with $d< n$) increases the per-candidate chance of landing in an escaping region compared to a whole-body search in
$\mathbb{R}^{n}$.

\begin{remark}
    \textup{
    Combining the compute--coverage advantage (\textit{i.e.}, more candidates per time step) and the dimensionality advantage (\textit{i.e.}, less needle-like escape regions) yields that distributed parallel continuation can increase the likelihood of producing an escaping candidate update within a fixed time budget, while ensuring constraint satifaction.
    }
\end{remark}

\begin{figure*}[!t]
\vspace{0.2cm}
    \center	
    \includegraphics[clip, trim=1cm 5cm 1cm 0cm, width=2\columnwidth]{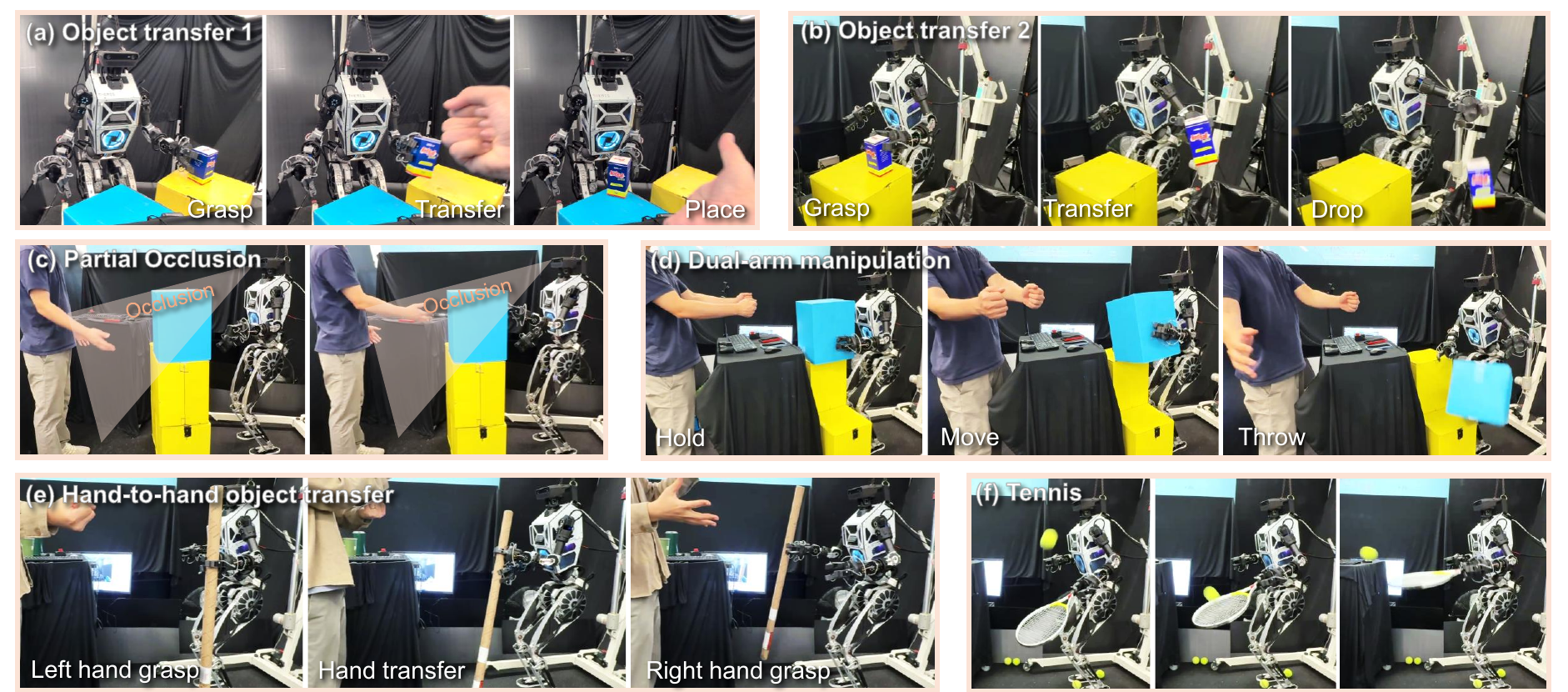}
    \caption{{\bfseries Snapshots of useful real-world tasks through MIRROR teleoperation.} }
    \label{fig:tasks}
    \vspace{-0.2cm}
\end{figure*}

\section{Results}
\label{sec:Results}

In this section, we evaluate the proposed MIRROR framework in both simulation and hardware environments.

\subsection{Validation Setup}

All experiments are conducted using the full pipeline described in Sec.~\ref{sec:approach}.
Simulation experiments are performed using the MuJoCo simulation backend. 
Hardware experiments are conducted on Westwood Robotics' \textit{THEMIS V2 Pro} humanoid robot with 40 actuated DOF. 

The simulation testing platform is backed by a Ubuntu 22 desktop equipped with an Intel Ultra 9 285K and an NVIDIA RTX 5090.
All CPU-based IK subproblems are formulated as constrained quadratic programs and solved using \texttt{OSQP}.
The GPU-accelerated distributed batch implementation is realized via a customized \texttt{ADMM} solver for parallel evaluation in \texttt{PyTorch}. The robot kinematics and dynamics evaluation uses a customized Python implementation of Spatial V2 \cite{featherstone2014rigid}.


\subsection{Robustness of Pose Estimation and Latency}

\begin{table*}[!t]
\centering
\caption{Ablation study on differential IK solving methods over long-horizon hand tracking.}
\label{tab:multistep_comparison}
\setlength{\tabcolsep}{5pt}
\renewcommand{\arraystretch}{1.15}
\scriptsize
\begin{tabular}{l c c c c  c c c c}
\toprule \toprule
\textbf{Method}
& \textbf{Batch $K$}
& $\boldsymbol{\eta}$
& \textbf{Solve Time [ms] $\downarrow$}
& $\boldsymbol{||\bar e_\textrm{hand,l}||}$ \textbf{[mm] $\downarrow$}
& $\boldsymbol{||\bar e_\textrm{hand,l}^\textrm{final}||}$ \textbf{[mm] $\downarrow$}
& \textbf{Self Collision}$\downarrow$
& \textbf{Singularity}$\downarrow$
& \textbf{Stagnation}$\downarrow$
\\
\midrule

Global IK (SQP)
& 1
& --
& $176.57\pm4.20$
& $12.44\pm3.12$
& $6.58\pm2.23$
& $42/250$
& $0/250$
& $0/250$
\\
\midrule

\greytext{Monolithic} QP
& 1
& --
& $0.58\pm0.05$
& $18.57\pm9.82$
& $16.23\pm4.12$
& $135/250$
& $17/250$
& $44/250$
\\
\midrule

Distributed QP
& 1
& --
& $0.94\pm0.06$
& $17.72\pm10.68$
& $10.44\pm4.56$
& $142/250$
& $19/250$
& $46/250$
\\
\midrule

\multirow{3}{*}{\shortstack[l]{Parallel Dist. QPs\\ {\scriptsize - Continuation \greentext{$\checkmark$}} \\ {\scriptsize - $V$ certificate \todo{$\times$}}}}
& $256$  & -- 
& $4.28\pm0.14$ 
& $16.61\pm7.95$ 
& $11.45\pm3.34$ 
& $68/250$ & $9/250$ & $45/250$ \\
\arrayrulecolor{gray!50}\cmidrule(l){2-9}\arrayrulecolor{black}
& $4096$ & -- 
& $4.08\pm0.06$ 
& $16.61\pm7.95$ 
& $11.45\pm3.34$ 
& $68/250$ & $9/250$ & $45/250$ \\
\midrule

\multirow{3}{*}{\shortstack[l]{  Parallel \greytext{Mono.} QPs\\ {\scriptsize - Continuation \greentext{$\checkmark$}} \\ {\scriptsize - $V$ certificate \greentext{$\checkmark$}}}}
& $256$  & 0.0005 
& $3.94\pm0.09$ 
& $16.61\pm7.86$ 
& $10.38\pm4.20$ 
& $52/250$ & $9/250$ & $45/250$ \\
\arrayrulecolor{gray!50}\cmidrule(l){2-9}\arrayrulecolor{black}
& $4096$ & 0.0005 
& $3.76\pm0.03$ 
& $16.78\pm7.33$ 
& $10.54\pm4.12$ 
& $47/250$ & $7/250$ & $31/250$ \\
\midrule

\multirow{3}{*}{\shortstack[l]{\textbf{Parallel Dist. QPs}\\ {\scriptsize - Continuation \greentext{$\checkmark$}} \\ {\scriptsize - $V$ certificate \greentext{$\checkmark$}}}}
& $256$  & $0.0005$ 
& $4.27\pm0.09$ 
& \cellcolor{Lavender!40} $16.63\pm7.88$ 
& \cellcolor{Lavender!40} $10.34\pm4.11$ 
& \cellcolor{Lavender!10} $\bf 41/250$ & \cellcolor{Lavender!10} $\bf 9/250$ &  \cellcolor{Lavender!10} $\bf 28/250$ \\
\arrayrulecolor{gray!50}\cmidrule(l){2-9}\arrayrulecolor{black}
& $4096$ & $0.0005$ 
& $4.10\pm0.04$ 
& \cellcolor{Lavender!40}$16.67\pm7.58$ 
& \cellcolor{Lavender!40}$10.30\pm4.45$ 
& \cellcolor{Lavender!20}$\bf 18/250$ & \cellcolor{Lavender!20}$\bf 7/250$ & \cellcolor{Lavender!20}$\bf 17/250$ \\
\arrayrulecolor{gray!50}\cmidrule(l){2-9}\arrayrulecolor{black}
& $4096$ & $0.005$ & 
$4.12\pm0.07$ 
& \cellcolor{Lavender!20} $18.57\pm6.89$ 
& \cellcolor{Lavender!20}$13.52\pm5.86$ 
& \cellcolor{Lavender!30}$\bf 10/250$ & \cellcolor{Lavender!30}$\bf 7/250$ & \cellcolor{Lavender!30}$\bf 11/250$ \\
\arrayrulecolor{gray!50}\cmidrule(l){2-9}\arrayrulecolor{black}
& $4096$ & $0.001$ 
& $4.14\pm0.07$ 
& \cellcolor{Lavender!10}$22.09\pm7.16$ 
& \cellcolor{Lavender!10}$15.65\pm4.41$ 
& \cellcolor{Lavender!40}$\bf 3/250$ & \cellcolor{Lavender!40}$\bf 5/250$ & \cellcolor{Lavender!40}$\bf 7/250$ \\

\bottomrule \bottomrule
\end{tabular}

\vspace{-0.15cm}
\begin{flushleft}
    {\footnotesize
 Lower is better ($\downarrow$). Heavier \colorbox{Lavender!30}{shade} means better results. Solve time is measured as the end-to-end duration of the entire IK problem, including Jacobian, CBF gradients, and batch overheads. Self-collision is determined based on CBF geometry in simulation instead of the actual robot geometry. Stagnation is counted when any joint reaches its joint limit and does not progress away from the joint limit (\textit{i.e.}, joint lock). \texttt{OSQP} setup: $\textrm{Max iter.} = 500$,  Tolerance $\textrm{eps}_\textrm{abs} = \textrm{eps}_\textrm{rel} = 1\text{e-}5$, Warm-start On; Custom \texttt{ADMM-QP} setup: $\textrm{Max iter.} = 50$, Tikhonov regularization $\sigma = 1\text{e-}6$, Penalty $\rho = 50$, Warm-start On.
}
\end{flushleft}

\vspace{-0.25cm}
\end{table*}

We evaluate the perception front-end by comparing tracking accuracy and smoothness against established motion capture systems. Human hand trajectories are recorded simultaneously using the proposed camera-based pipeline (Vision), a Meta Quest 3 VR headset (VR), and an OptiTrack system as ground truth, with task-space keypoints aligned in a common reference frame. As shown in Fig.~\ref{fig:comparison}, VR tracking closely follows OptiTrack with minor dropouts, while raw vision estimates exhibit higher-frequency noise. After filtering and robustification, the vision trajectories become substantially smoother and preserve overall motion trends and peak events, maintaining strong temporal consistency with the ground truth. These results indicate that the filtered vision pipeline provides relatively accurate and sufficiently smooth tracking for motion retargeting. Moreover, unlike the VR setup, which primarily tracks end-effectors, the vision-based approach supports full-body pose estimation in a lightweight fashion. 

\begin{figure}[!h]
\vspace{-0.2cm}
    \center	
    \includegraphics[clip, trim=0cm 0cm 0cm 0cm, width=1\columnwidth]{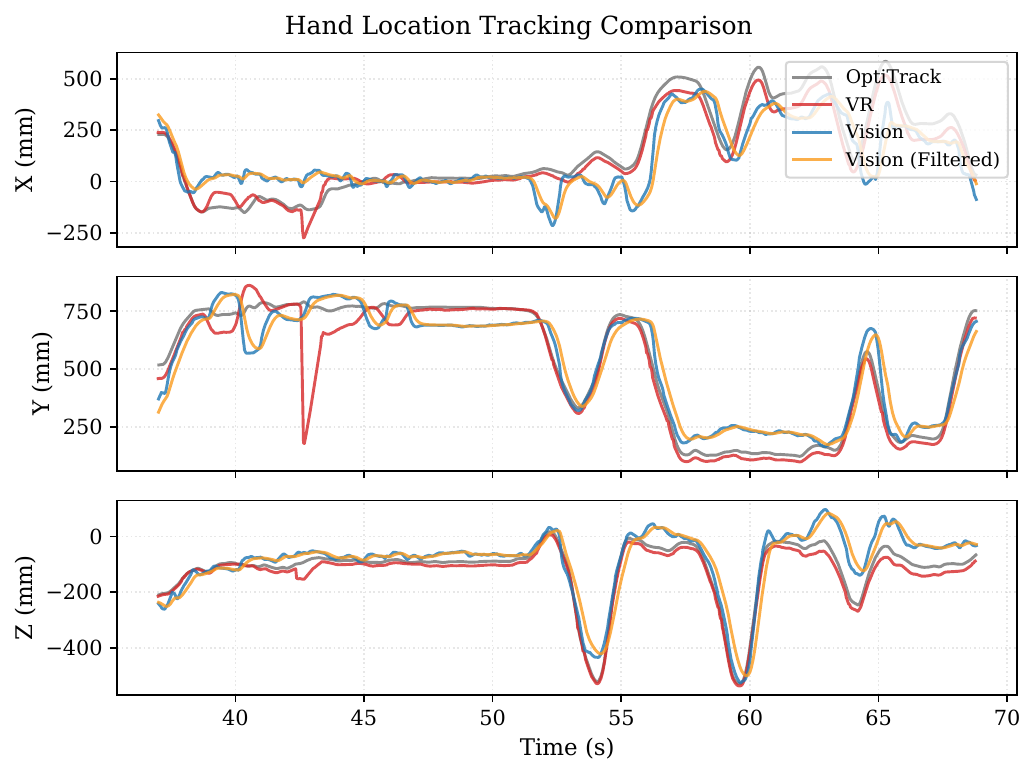}
    \caption{{\bfseries Comparison of motion capture devices tracking hand location.} }
    \label{fig:comparison}
    \vspace{-0.5cm}
\end{figure}

\subsection{Inverse Kinematics Ablation Study}

We compare the proposed continuation-based distributed parallel differential IK against several ablated baselines in simulation. 
We ablate key features, including parallel solver, task-space continuation, and Lyapunov-based progress certificate. 
Table~\ref{tab:multistep_comparison} evaluates IK variants over 50-step closed-loop tracking of moving hand targets for both arms across 250 randomized trials in simulation. While the global IK achieves the lowest tracking error, its high computational cost ($\sim$177 ms) makes it unsuitable for real-time control; in contrast, single-instance monolithic and distributed QPs are fast ($<$1 ms) but exhibit frequent self-collision, singularity violations, and stagnation due to basin trapping over long horizons. Although CBF constraints are explicit in QP, closed-loop invariance is not guaranteed, and self-collision can still arise in practice due to large step sizes, tracking performance, and controller delay.  
We also intentionally generate hand targets near or across CBF boundaries to stress-test the solver.

On the other hand, parallel continuation substantially improves robustness without sacrificing millisecond-level performance ($\sim$4 ms), and incorporating the Lyapunov-based progress certificate further reduces violations and stagnation, emphasizing a good trade-off between safety and accuracy. 
In the proposed approach, increasing the batch size $K$ consistently reduces CBF violations and stagnation events, empirically confirming that evaluating more continuation candidates improves the likelihood of selecting a progress-certified escaping update, as outlined by Theorem~\ref{thm:escape_prob_pc}. Meanwhile, tuning the progress threshold $\eta$ exposes a clear robustness–accuracy trade-off: larger $\eta$ enforces stronger Lyapunov descent and improves safety and feasibility, but at the cost of increased tracking error due to more conservative updates.

\subsection{Hardware Pipeline Latency}

We measure the end-to-end system latency of the MIRROR pipeline on hardware with 3 setups. A latency breakdown is presented in Fig.~\ref{fig:latency}, where the total latency from the perception front-end to the robot receiving the joint-space commands averaged at \textbf{54.9 ms} with the tethered desktop mode.  The completely untethered setups with onboard computing units are also real-time feasible, with a tradeoff of perception robustness and batch number in IK. We also measured the true End-to-End latency from human motion to robot motion by decomposing camera frames from recorded experiment videos, yielding a consistent latency of \textbf{$\sim$170 ms} (tethered PD mode) and \textbf{$\sim$250 ms} (tethered WBC mode), which surpasses existing major camera-based (tethered) teleoperation systems such as H2O\cite{he2024learning} ($\sim373$ ms) and HumanPlus \cite{fu2024humanplus} ($\sim340$ ms), and is comparable to other VR/MoCap-based pipelines, as reported in \cite{xiong2026extremcontrol}.

\begin{figure}[!t]
\vspace{-0.3cm}
    \center	
    \includegraphics[clip, trim=0cm 0cm 0cm 0cm, width=1\columnwidth]{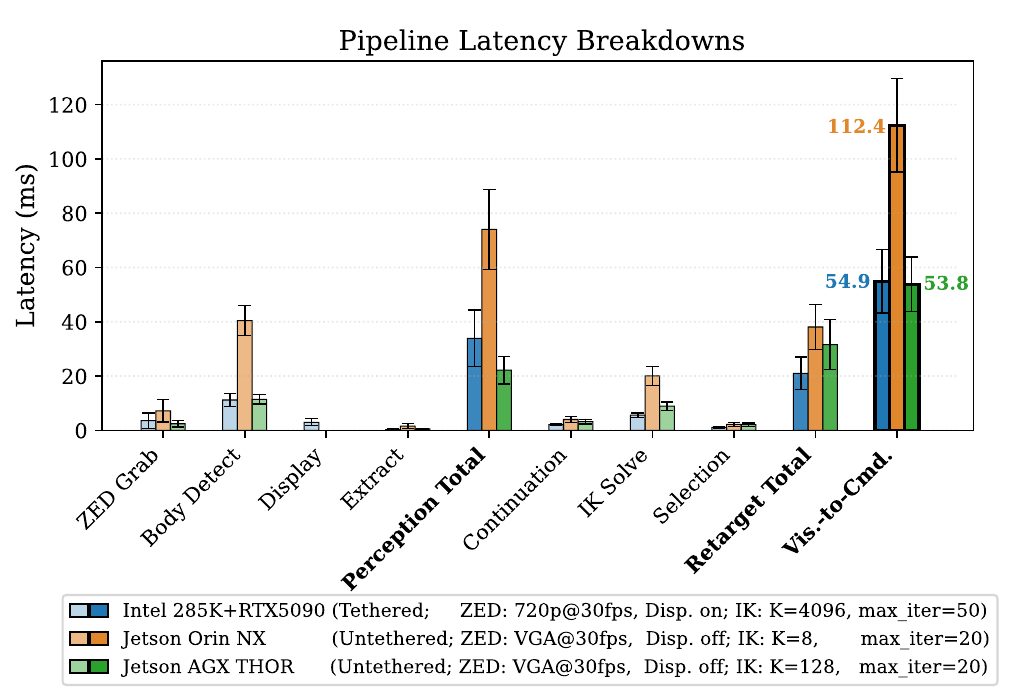}
    \caption{{\bfseries Latency breakdown with different hardware setups.} Only the main components of each module are presented here; overhead and communication latencies are only counted in the "Module Total" latencies.}
    \label{fig:latency}
    \vspace{-0.2cm}
\end{figure}

\subsection{Real-World Teleoperation Performance}
\begin{figure}[!t]
\vspace{0.0cm}
    \center	
    \includegraphics[clip, trim=0cm 0.3cm 0cm 0cm, width=1\columnwidth]{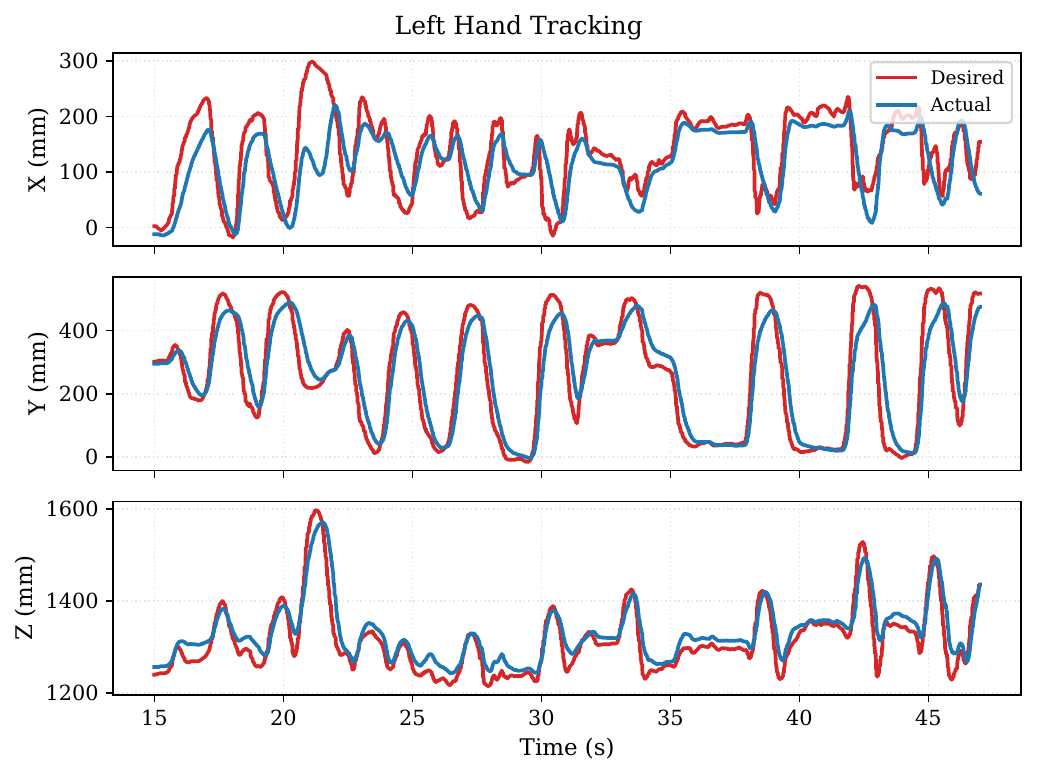}
    \caption{{\bfseries Left hand tracking performance on hardware.} Hardware is using an optimization-based whole-body controller to execute joint-level commands and balancing tasks.}
    \label{fig:tracking}
    \vspace{-0.4cm}
\end{figure}

Finally, we evaluate the proposed system in real-world teleoperation tasks using the THEMIS humanoid platform. The demonstrator performs a set of upper-body manipulation motions while the retargeted motions are executed in real time using the full pipeline, as shown in Fig. \ref{fig:mirroring} and Fig. \ref{fig:tasks}.
We report task-space tracking errors of the left hand from a representative trial, as shown in Fig. \ref{fig:tracking}. 
It is worth noting that tracking performance also depends on the underlying low-level controller; since MIRROR is agnostic to this layer, it can be paired with more advanced controllers to potentially achieve improved tracking performance.

\section{Conclusion and Discussion}
\label{sec:Conclusion}

In this work, we presented MIRROR, a real-time visual motion imitation and teleoperation framework that integrates stereo-based full-body perception with a GPU-accelerated, continuation-based parallel differential IK solver. The proposed method improves robustness to basin-dependent stagnation and constraint-induced trapping while preserving millisecond-level solve times, enabling stable, low-latency humanoid teleoperation in real-time. 

Although this work shows primarily upper-body motion retargeting, the distributed parallel formulation scales naturally to full-body by incorporating additional limb chains without additional computation and latency, thanks to parallelism in distributed subproblems. However, coordination constraints may require additional coupling terms/consensus.

\balance
\bibliographystyle{ieeetr}
\bibliography{reference.bib}

@article{gu2025humanoid,
  title={Humanoid locomotion and manipulation: Current progress and challenges in control, planning, and learning},
  author={Gu, Zhaoyuan and Li, Junheng and Shen, Wenlan and Yu, Wenhao and Xie, Zhaoming and McCrory, Stephen and Cheng, Xianyi and Shamsah, Abdulaziz and Griffin, Robert and Liu, C Karen and others},
  journal={arXiv preprint arXiv:2501.02116},
  year={2025}
}

@article{zago20203d,
  title={3D tracking of human motion using visual skeletonization and stereoscopic vision},
  author={Zago, Matteo and Luzzago, Matteo and Marangoni, Tommaso and De Cecco, Mariolino and Tarabini, Marco and Galli, Manuela},
  journal={Frontiers in bioengineering and biotechnology},
  volume={8},
  pages={181},
  year={2020},
  publisher={Frontiers Media SA}
}

@article{del2016implementing,
  title={Implementing torque control with high-ratio gear boxes and without joint-torque sensors},
  author={Del Prete, Andrea and Mansard, Nicolas and Ramos, Oscar E and Stasse, Olivier and Nori, Francesco},
  journal={International Journal of Humanoid Robotics},
  volume={13},
  number={01},
  pages={1550044},
  year={2016},
  publisher={World Scientific}
}

@software{pink,
  title = {{Pink: Python inverse kinematics based on Pinocchio}},
  author = {Caron, Stéphane and De Mont-Marin, Yann and Budhiraja, Rohan and Bang, Seung Hyeon and Domrachev, Ivan and Nedelchev, Simeon and Du, Peter and Escande, Adrien and Vaillant, Joris and Wingo, Bruce and Patapati, Santosh},
  license = {Apache-2.0},
  url = {https://github.com/stephane-caron/pink},
  version = {4.0.0},
  year = {2026}
}

@article{kobayashi2023mirror,
  title={Mirror-Descent Inverse Kinematics with Box-constrained Joint Space},
  author={Kobayashi, Taisuke and Jin, Takanori},
  journal={IFAC-PapersOnLine},
  volume={56},
  number={2},
  pages={300--305},
  year={2023},
  publisher={Elsevier}
}

@inproceedings{sundaralingam2023curobo,
  title={Curobo: Parallelized collision-free robot motion generation},
  author={Sundaralingam, Balakumar and Hari, Siva Kumar Sastry and Fishman, Adam and Garrett, Caelan and Van Wyk, Karl and Blukis, Valts and Millane, Alexander and Oleynikova, Helen and Handa, Ankur and Ramos, Fabio and others},
  booktitle={2023 IEEE International Conference on Robotics and Automation (ICRA)},
  pages={8112--8119},
  year={2023},
  organization={IEEE}
}

@article{ames2022ikflow,
  title={Ikflow: Generating diverse inverse kinematics solutions},
  author={Ames, Barrett and Morgan, Jeremy and Konidaris, George},
  journal={IEEE Robotics and Automation Letters},
  volume={7},
  number={3},
  pages={7177--7184},
  year={2022},
  publisher={IEEE}
}

@article{ma2020coupled,
  title={Coupled control systems: Periodic orbit generation with application to quadrupedal locomotion},
  author={Ma, Wen-Loong and Csomay-Shanklin, Noel and Ames, Aaron D},
  journal={IEEE Control Systems Letters},
  volume={5},
  number={3},
  pages={935--940},
  year={2020},
  publisher={IEEE}
}

@article{ames2016control,
  title={Control barrier function based quadratic programs for safety critical systems},
  author={Ames, Aaron D and Xu, Xiangru and Grizzle, Jessy W and Tabuada, Paulo},
  journal={IEEE Transactions on Automatic Control},
  volume={62},
  number={8},
  pages={3861--3876},
  year={2016},
  publisher={IEEE}
}

@article{xiong2026extremcontrol,
  title={ExtremControl: Low-Latency Humanoid Teleoperation with Direct Extremity Control},
  author={Xiong, Ziyan and Fang, Lixing and Huang, Junyun and Yamazaki, Kashu and Zhang, Hao and Gan, Chuang},
  journal={arXiv preprint arXiv:2602.11321},
  year={2026}
}

@article{fu2024humanplus,
  title={Humanplus: Humanoid shadowing and imitation from humans},
  author={Fu, Zipeng and Zhao, Qingqing and Wu, Qi and Wetzstein, Gordon and Finn, Chelsea},
  journal={arXiv preprint arXiv:2406.10454},
  year={2024}
}

@inproceedings{colin2023whole,
  title={Whole-body dynamic telelocomotion: A step-to-step dynamics approach to human walking reference generation},
  author={Colin, Guillermo and Byrnes, Joseph and Sim, Youngwoo and Wensing, Patrick M and Ramos, Joao},
  booktitle={2023 IEEE-RAS 22nd International Conference on Humanoid Robots (Humanoids)},
  pages={1--8},
  year={2023},
  organization={IEEE}
}

@inproceedings{amatucci2024accelerating,
  title={Accelerating model predictive control for legged robots through distributed optimization},
  author={Amatucci, Lorenzo and Turrisi, Giulio and Bratta, Angelo and Barasuol, Victor and Semini, Claudio},
  booktitle={2024 IEEE/RSJ International Conference on Intelligent Robots and Systems (IROS)},
  pages={12734--12741},
  year={2024},
  organization={IEEE}
}

@article{yasutake2025hjcd,
  title={HJCD-IK: GPU-Accelerated Inverse Kinematics through Batched Hybrid Jacobian Coordinate Descent},
  author={Yasutake, Cael and Kingston, Zachary and Plancher, Brian},
  journal={arXiv preprint arXiv:2510.07514},
  year={2025}
}

@article{giamou2022convex,
  title={Convex iteration for distance-geometric inverse kinematics},
  author={Giamou, Matthew and Mari{\'c}, Filip and Rosen, David M and Peretroukhin, Valentin and Roy, Nicholas and Petrovi{\'c}, Ivan and Kelly, Jonathan},
  journal={IEEE Robotics and Automation Letters},
  volume={7},
  number={2},
  pages={1952--1959},
  year={2022},
  publisher={IEEE}
}

@book{bartle2014elements,
  title={The elements of integration and Lebesgue measure},
  author={Bartle, Robert G},
  year={2014},
  publisher={John Wiley \& Sons}
}

@inproceedings{he2025novel,
  title={A Novel Telelocomotion Framework with CoM Estimation for Scalable Locomotion on Humanoid Robots},
  author={He, An-Chi and Li, Junheng and Park, Jungsoo and Kolt, Omar and Beiter, Ben and Leonessa, Alexander and Nguyen, Quan and Hamed, Kaveh Akbari},
  booktitle={2025 IEEE International Conference on Robotics and Automation (ICRA)},
  pages={5615--5621},
  year={2025},
  organization={IEEE}
}

@article{darvish2023teleoperation,
  title={Teleoperation of humanoid robots: A survey},
  author={Darvish, Kourosh and Penco, Luigi and Ramos, Joao and Cisneros, Rafael and Pratt, Jerry and Yoshida, Eiichi and Ivaldi, Serena and Pucci, Daniele},
  journal={IEEE Transactions on Robotics},
  volume={39},
  number={3},
  pages={1706--1727},
  year={2023},
  publisher={IEEE}
}

@inproceedings{bishop2024relu,
  title={Relu-qp: A gpu-accelerated quadratic programming solver for model-predictive control},
  author={Bishop, Arun L and Zhang, John Z and Gurumurthy, Swaminathan and Tracy, Kevin and Manchester, Zachary},
  booktitle={2024 IEEE International Conference on Robotics and Automation (ICRA)},
  pages={13285--13292},
  year={2024},
  organization={IEEE}
}

@article{jeon2024cusadi,
  title={Cusadi: A gpu parallelization framework for symbolic expressions and optimal control},
  author={Jeon, Se Hwan and Hong, Seungwoo and Lee, Ho Jae and Khazoom, Charles and Kim, Sangbae},
  journal={IEEE Robotics and Automation Letters},
  volume={10},
  number={2},
  pages={899--906},
  year={2024},
  publisher={IEEE}
}

@inproceedings{alvarez2025real,
  title={Real-time whole-body control of legged robots with model-predictive path integral control},
  author={Alvarez-Padilla, Juan and Zhang, John Z and Kwok, Sofia and Dolan, John M and Manchester, Zachary},
  booktitle={2025 IEEE International Conference on Robotics and Automation (ICRA)},
  pages={14721--14727},
  year={2025},
  organization={IEEE}
}

@article{araujo2025retargeting,
  title={Retargeting matters: General motion retargeting for humanoid motion tracking},
  author={Araujo, Joao Pedro and Ze, Yanjie and Xu, Pei and Wu, Jiajun and Liu, C Karen},
  journal={arXiv preprint arXiv:2510.02252},
  year={2025}
}

@software{Zakka_Mink_Python_inverse_2026,
  author = {Zakka, Kevin},
  title = {{Mink: Python inverse kinematics based on MuJoCo}},
  year = {2026},
  month = feb,
  version = {1.1.0},
  url = {https://github.com/kevinzakka/mink},
  license = {Apache-2.0}
}

@article{yang2025omniretarget,
  title={Omniretarget: Interaction-preserving data generation for humanoid whole-body loco-manipulation and scene interaction},
  author={Yang, Lujie and Huang, Xiaoyu and Wu, Zhen and Kanazawa, Angjoo and Abbeel, Pieter and Sferrazza, Carmelo and Liu, C Karen and Duan, Rocky and Shi, Guanya},
  journal={arXiv preprint arXiv:2509.26633},
  year={2025}
}

@inproceedings{xue2025full,
  title={Full-order sampling-based mpc for torque-level locomotion control via diffusion-style annealing},
  author={Xue, Haoru and Pan, Chaoyi and Yi, Zeji and Qu, Guannan and Shi, Guanya},
  booktitle={2025 IEEE International Conference on Robotics and Automation (ICRA)},
  pages={4974--4981},
  year={2025},
  organization={IEEE}
}

@inproceedings{todorov2012mujoco,
  title={Mujoco: A physics engine for model-based control},
  author={Todorov, Emanuel and Erez, Tom and Tassa, Yuval},
  booktitle={2012 IEEE/RSJ international conference on intelligent robots and systems},
  pages={5026--5033},
  year={2012},
  organization={IEEE}
}

@article{ze2025twist,
  title={Twist: Teleoperated whole-body imitation system},
  author={Ze, Yanjie and Chen, Zixuan and Ara{\'u}jo, Joao Pedro and Cao, Zi-ang and Peng, Xue Bin and Wu, Jiajun and Liu, C Karen},
  journal={arXiv preprint arXiv:2505.02833},
  year={2025}
}

@article{dafarra2024icub3,
  title={icub3 avatar system: Enabling remote fully immersive embodiment of humanoid robots},
  author={Dafarra, Stefano and Pattacini, Ugo and Romualdi, Giulio and Rapetti, Lorenzo and Grieco, Riccardo and Darvish, Kourosh and Milani, Gianluca and Valli, Enrico and Sorrentino, Ines and Viceconte, Paolo Maria and others},
  journal={Science Robotics},
  volume={9},
  number={86},
  pages={eadh3834},
  year={2024},
  publisher={American Association for the Advancement of Science}
}

@inproceedings{bertrand2024high,
  title={High-speed and impact resilient teleoperation of humanoid robots},
  author={Bertrand, Sylvain and Penco, Luigi and Anderson, Dexton and Calvert, Duncan and Roy, Valentine and McCrory, Stephen and Mohammed, Khizar and Sanchez, Sebastian and Griffith, Will and Morfey, Steve and others},
  booktitle={2024 IEEE-RAS 23rd International Conference on Humanoid Robots (Humanoids)},
  pages={189--196},
  year={2024},
  organization={IEEE}
}

@article{he2024omnih2o,
  title={Omnih2o: Universal and dexterous human-to-humanoid whole-body teleoperation and learning},
  author={He, Tairan and Luo, Zhengyi and He, Xialin and Xiao, Wenli and Zhang, Chong and Zhang, Weinan and Kitani, Kris and Liu, Changliu and Shi, Guanya},
  journal={arXiv preprint arXiv:2406.08858},
  year={2024}
}

@inproceedings{li2025clone,
  title={Clone: Closed-loop whole-body humanoid teleoperation for long-horizon tasks},
  author={Li, Yixuan and Lin, Yutang and Cui, Jieming and Liu, Tengyu and Liang, Wei and Zhu, Yixin and Huang, Siyuan},
  booktitle={9th Annual Conference on Robot Learning},
  year={2025}
}

@inproceedings{he2024learning,
  title={Learning human-to-humanoid real-time whole-body teleoperation},
  author={He, Tairan and Luo, Zhengyi and Xiao, Wenli and Zhang, Chong and Kitani, Kris and Liu, Changliu and Shi, Guanya},
  booktitle={2024 IEEE/RSJ International Conference on Intelligent Robots and Systems (IROS)},
  pages={8944--8951},
  year={2024},
  organization={IEEE}
}

@inproceedings{khazoom2022humanoid,
  title={Humanoid arm motion planning for improved disturbance recovery using model hierarchy predictive control},
  author={Khazoom, Charles and Kim, Sangbae},
  booktitle={2022 International Conference on Robotics and Automation (ICRA)},
  pages={6607--6613},
  year={2022},
  organization={IEEE}
}

@book{featherstone2014rigid,
  title={Rigid body dynamics algorithms},
  author={Featherstone, Roy},
  year={2014},
  publisher={Springer}
}

\end{document}